\documentclass{article}
\usepackage{ijcai25}

\usepackage{times}
\usepackage{soul}
\usepackage{url}
\usepackage[hidelinks]{hyperref}
\usepackage[utf8]{inputenc}
\usepackage[small]{caption}
\usepackage{graphicx}
\usepackage{amsmath}
\usepackage{amsthm}
\usepackage{booktabs}
\usepackage{algorithm}
\usepackage[switch]{lineno}

\newtheorem{theorem}{Theorem}

\usepackage{amssymb}
\usepackage{booktabs}
\usepackage{bbding,utfsym}
\usepackage{subcaption}
\usepackage[english]{babel}
\usepackage{multirow}
\usepackage{tabularx}
\usepackage{adjustbox}
\usepackage{threeparttable}
\usepackage{algpseudocode}

\title{Adaptive Gradient Learning for Spiking Neural Networks by \\ Exploiting Membrane Potential Dynamics}

\author{
Jiaqiang Jiang$^1$\and
Lei Wang$^1$\and
Runhao Jiang$^2$\and
Jing Fan$^1$\And
Rui Yan$^1$\thanks{Corresponding author}\\
\affiliations
$^1$College of Computer Science and Technology, Zhejiang University of Technology, Hangzhou, China\\
$^2$College of Computer Science and Technology, Zhejiang University, Hangzhou, China\\
\emails
\{jqjiang, LeiWang23\}@zjut.edu.cn,
RhJiang@zju.edu.cn, \{fanjing, ryan\}@zjut.edu.cn}

\usepackage{bibentry}

\begin{document}
\maketitle

\begin{abstract}
Brain-inspired spiking neural networks (SNNs) are recognized as a promising avenue for achieving efficient, low-energy neuromorphic computing. Recent advancements have focused on directly training high-performance SNNs by estimating the approximate gradients of spiking activity through a continuous function with constant sharpness, known as surrogate gradient (SG) learning. However, as spikes propagate among neurons, the distribution of membrane potential dynamics (MPD) will deviate from the gradient-available interval of fixed SG, hindering SNNs from searching the optimal solution space. To maintain the stability of gradient flows, SG needs to align with evolving MPD. Here, we propose adaptive gradient learning for SNNs by exploiting MPD, namely MPD-AGL. It fully accounts for the underlying factors contributing to membrane potential shifts and establishes a dynamic association between SG and MPD at different timesteps to relax gradient estimation, which provides a new degree of freedom for SG learning. Experimental results demonstrate that our method achieves excellent performance at low latency. Moreover, it increases the proportion of neurons that fall into the gradient-available interval compared to fixed SG, effectively mitigating the gradient vanishing problem.
\end{abstract}

\section{Introduction}
As a new paradigm with biological plausibility and computational efficiency, spiking neural networks (SNNs) achieve unique sparse coding and asynchronous information processing by modeling the spike firing and temporal dynamics of biological neurons. Instead of artificial neural networks (ANNs) that work with continuous activation and multiply-and-accumulate (MAC) operations, SNNs operate with threshold firing and accumulate (AC) operations, which allow low-latency inference and low-power computation on neuromorphic hardware \cite{akopyan2015truenorth,ma2024darwin3,davies2018loihi,pei2019towards}. Nowadays, with the development of SNNs, it has exhibited high potential in many applications, such as image classification \cite{liang2024event,yang2024multi}, object detection \cite{wang2025eas,wang2025adaptive}, reinforcement learning \cite{qin2023low,qin2025grsn}, etc. Backpropagation-based learning is a favorable methodology for training high-performance SNNs \cite{huh2018gradient,dampfhoffer2023backpropagation}. Nevertheless, the discontinuous nature of spiking neurons hinders the direct application of gradient descent in SNNs. To tackle the non-differentiability of spike activity, surrogate gradient (SG) methods employ a smooth curve to distribute the gradient of output signals into a group of analog items in temporal neighbors \cite{zhang2020temporal}. Unfortunately, as spikes propagate in the spatio-temporal domain (STD), the distribution of membrane potential will shift and may not align with the gradient-available interval of fixed SG, leading to gradient vanishing or mismatch problems \cite{guo2022recdis}.

\begin{figure*}[htbp]
    \centering
    \includegraphics[width=0.94\linewidth]{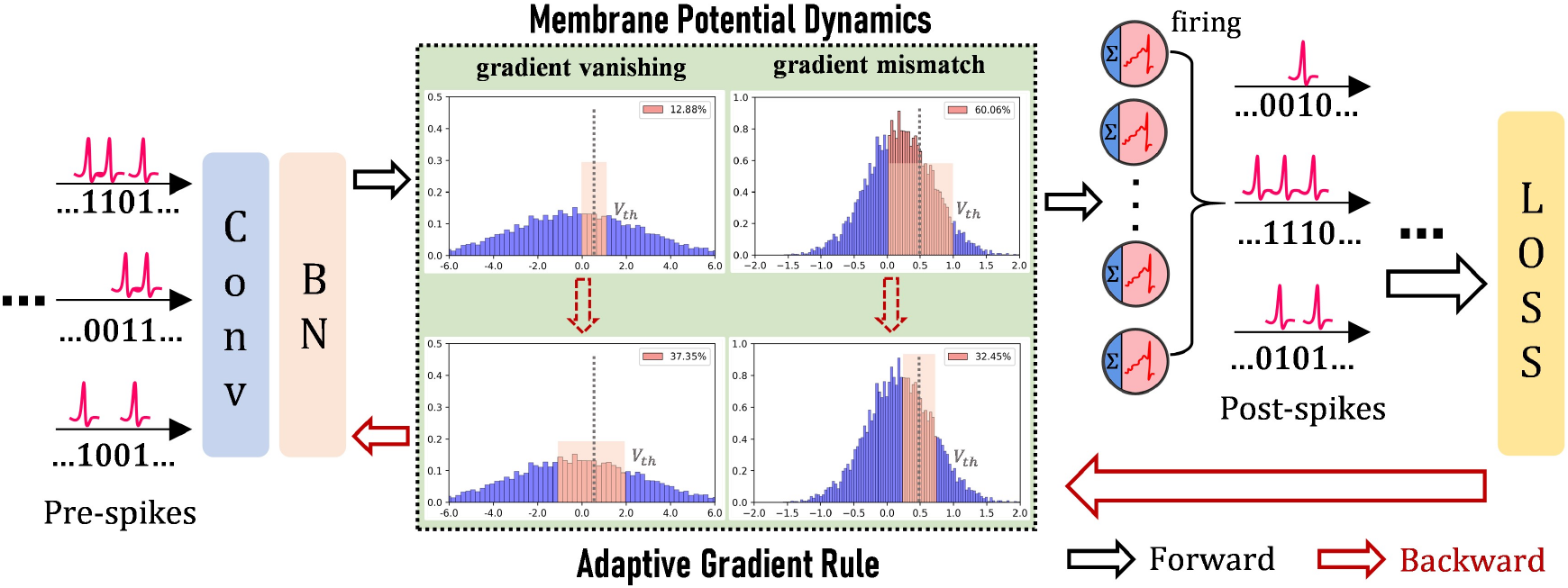}
    \caption{The overall framework of MPD-AGL. Pre-spikes are passed through the convolutional and normalization layers and then injected into spiking neurons to compute membrane potentials and fire spikes. The distribution of evolving MPD in forward propagation may not align with the fixed SG, leading to gradient vanishing or mismatch problems in backward propagation. Instead, the proposed adaptive gradient rule can synchronously adjust the width of SG to respond to evolving MPD during the entire timestep.}
    \label{Fig::overall_framework}
\end{figure*}

In Fig. \ref{Fig::overall_framework}, the main reason for gradient vanishing or mismatch problems is that the overlap area between the evolving membrane potential dynamics (MPD) and the gradient-available interval of fixed SG becomes too narrow or too wide. On the one hand, the limited overlap area causes many membrane potentials to fall into the area with zero approximate derivatives, leading to gradient propagation blockage. On the other hand, when the overlap area saturates, neurons contribute many inaccurate approximate gradients, enlarging the error with true gradients. To match SG and MPD, two groups of methods have been developed: (1) membrane potential regulation and (2) SG optimization. Membrane potential regulation methods redistribute the membrane potential before firing \cite{guo2022reducing} or define a distribution loss to rectify it \cite{guo2022recdis,wang2025potential}, aiming to balance the distribution to minimize the undesired shifts, but this increases the inference burden or requires more parameters and computations. By contrast, SG optimization methods update the SG by capturing the direction of accurate gradients that can automatically calibrate the SG sharpness in response to MPD for better gradient estimation \cite{wang2023adaptive,wang2025potential}. However, most of these methods either focus only on regulating membrane potentials or only on optimizing SG, ignoring their correlation, which cannot effectively control their alignment. Moreover, the lack of comprehensive analysis regarding the causes of membrane potential shifts leaves room for improvement in these methods.

The reason for the membrane potential shifts and how to optimize SG to align with the evolving MPD in SNN learning are our main concerns. In this work, we propose an adaptive gradient learning algorithm for SNNs by exploiting MPD. Specifically, we realized that the affine transformation in normalization layers would force the pre-synaptic input to deviate from the desired distribution, affecting the distribution of MPD, which is the main cause of membrane potential shifts. Considering the influence of affine transformation, we derive the specific distribution of MPD at different timesteps during forward propagation and accordingly design a correlation function between SG and MPD to dynamically optimize SG, capturing the evolving MPD. The overall framework of our method is illustrated in Fig. \ref{Fig::overall_framework}. In summary, the main contributions of this work can be summarized as follows:
    \begin{itemize}
        \item We provide a new perspective for understanding the membrane potential shifts in SNN forward propagation by analyzing the effect of learnable affine transformation in the normalization layers on the distribution of MPD.
        
        \item We propose an adaptive gradient rule that synchronously adjusts the gradient-available interval of SG in response to the distribution of membrane potentials at different timesteps, aligning with the evolving MPD.
        
        \item Extensive experiments on four datasets CIFAR10, CIFAR100, CIFAR10-DVS, and Tiny-ImageNet show that our method overwhelmingly outperforms existing advanced SG optimization methods. Moreover, MPD-AGL consumes only 5.2$\%$ energy of ANN for a single inference at ultra-low latency $T=2$.
    \end{itemize}

\section{Related Work}
\subsection{Direct Training of SNNs}
With the introduction of spatio-temporal backpropagation and approximate derivatives of spike activity \cite{wu2018spatio,wu2019direct}, direct training of SNNs has ushered in a new opportunity. \cite{zheng2021going} proposed threshold-dependent batch normalization to balance the input stimulus and neuronal threshold, which extended SNNs to a deeper structure. \cite{yao2023attention,lee2025spiking} incorporated the attention mechanism to estimate the saliency of different domains, helping SNNs focus on important features. \cite{fang2021incorporating,yao2022glif} developed neuronal variants to learn membrane-related parameters, expanding the expressiveness of SNNs. \cite{deng2022temporal,guo2022loss} designed loss functions to regulate the distribution of spikes and membrane potentials along the temporal dimension to more accurately align the learning gradients.

\subsection{Gradient Alignment}
An essential component of SG learning is the suitable gradient flow \cite{zenke2021remarkable}. To alleviate the problem of fixed SG not aligned with evolving MPD, \cite{guo2022reducing} designed a membrane potential rectifier to redistribute potentials closer to the spikes. \cite{guo2022recdis} introduced three regularization losses to penalize three undesired shifts of MPD. \cite{wang2025potential} quantified the inconsistency between actual distributions and targets, which was integrated into the overall network loss for joint optimization. Optimizing SG is another appealing approach. \cite{guo2022loss} approximated the gradient of spike activity by a differentiable asymptotic function evolving continuously, bridging the gap between pseudo and natural derivatives. \cite{che2022differentiable} proposed a differentiable gradient search for parallel optimization of local SG. \cite{lian2023learnable} proposed a learnable SG to unlock the width limitation of SG. \cite{wang2023adaptive} learned the accurate gradients of loss landscapes adaptively by fusing the learnable relaxation degree into a prototype network with random spike noise. \cite{wang2025potential} proposed a parametric SG strategy that can be iteratively updated. Considering the lack of synergy between these methods in matching SG and MPD, this motivates us to explore their correlations to maximize matching optimization.

\section{Preliminary}
\subsection{Spiking Neural Model}
Based on the essential electrophysiological properties of biological neurons, the leaky integrate-and-fire (LIF) model simulates the electrical activity of neurons in a simplified mathematical form, widely used in SNNs as the basis unit. For computational tractability, \cite{wu2019direct} used the Euler formula to translate LIF into an iterative expression, the membrane potential evolves according to
\begin{align}
    I_i^n(t) &= \sum_{j=1}^{l(n-1)} w_{ij}^n S_j^{n-1}(t), \\
    V_i^n(t) &= \tau V_i^n(t-1)(1 - S_i^n(t-1))+ I_i^n(t), \label{Eq::iterativeLIF_V} \\
    S_i^n(t) &= \Theta(V_i^n(t)) = \begin{cases} 1, & V_i^n(t) \ge V_{th} \\ 0, & otherwise \end{cases} \label{Eq::activation_function}
\end{align}
where the superscript $n$, subscripts $i$ and $t$ denote the $n$-th layer, the $i$-th neuron and the $t$-th timestep, respectively. $l(n-1)$ denotes the number of neurons in the $(n-1)$-th layer. $w_{ij}^n$ denotes the synapse weight from the $j$-th neuron in the $(n-1)$-th layer to the $i$-th neuron in the $n$-th layer. $I$, $V$, and $S$ denote the pre-synaptic input, the membrane potential, and the binary spiking output of neurons, respectively. $\tau$ is the decay factor. $V_{th}$ is the firing threshold.

\subsection{Threshold-dependent Batch Normalization}
There are some drawbacks to directly applying BN techniques in SNNs due to the inherent temporal dynamics of spiking neurons \cite{wu2019direct}. To retain the advantages of BN in the channel dimension and capture the temporal dimension of SNN, \cite{zheng2021going} proposed threshold-dependent BN (tdBN), which normalized the pre-synaptic input $I$ to the distribution of $N(0,(\alpha V_{th})^2)$ instead of $N(0,1)$. Let $I^t_c$ represent the $c$-th channel feature maps of $I(t)$, then $I_c = (I^1_c, I^2_c, . . . , I^T_c)$ will be normalized as
\begin{align}
    \hat{I_c} &= \frac{\alpha V_{th}(I_c - \mathbb{E}[I_c])}{\sqrt{\mathbb{VAR}[I_c] + \epsilon}}, \quad \text{~// normalize} \\
    \label{Eq::tdBN_affine} \Bar{I_c} &= \gamma_c \hat{I_c} + \beta_c, \quad \quad \quad \text{// scale and shift}
\end{align}
where $\mathbb{E}[I_c]$ and $\mathbb{VAR}[I_c]$ denote the expectation and variance of $I_c$, which are computed over the Mini-Batch. $\epsilon$ is a tiny constant. The hyperparameter $\alpha$ is to prevent overfire or underfire, normally set to 1 \cite{zheng2021going}. The pair of learnable parameters $\gamma_c$ and $\beta_c$ are initial to 1 and 0, for scaling and shifting the normalized $\hat{I_c}$.

\subsection{Surrogate Gradient of SNNs}
In Eq. \ref{Eq::activation_function}, the activation function $\Theta(\cdot)$ of SNNs is a Heaviside step function. The derivative of output signals $\frac{\partial S}{\partial V}$ tends to infinity at the firing threshold $V_{th}$ and zeros otherwise, i.e. Dirac function. SG learning allows gradient information to be backpropagated layer-wise along STD, which lays the foundation for developing general SNNs. In this work, we employ the rectangular SG \cite{wu2018spatio}, which is defined as
\begin{equation}
    \frac{\partial S^n_i(t)}{\partial V^n_i(t)} \thickapprox h(V^n_i(t)) = \frac{1}{\kappa}sign(|V^n_i(t)-V_{th}| < \frac{\kappa}{2}),
\end{equation}
where hyperparameter $\kappa$ controls the width of $h(\cdot)$ to ensure it integrates to 1, normally set to 1 \cite{wu2018spatio,wu2019direct}. The gradient $\frac{1}{\kappa}$ is available when the membrane potential $V^n_i(t)$ falls within the interval $[V_{th} - \frac{\kappa}{2}, V_{th} + \frac{\kappa}{2}]$.

\section{Method}
In this section, we introduce the MPD-AGL algorithm in detail and the overall training procedure.

\subsection{Rethinking Pre-synaptic Input Distribution} \label{Sec::rethinking_inputDistribution}
To maintain the representation capacity of the layer, BN layers will normally perform a learnable affine transformation of the normalized pre-activations (Eq. \ref{Eq::tdBN_affine}). As shown in Fig. \ref{Fig::tdBN_param}, the learnable parameters $\gamma_c$ and $\beta_c$ will evolve, and their discrepancy grows more pronounced during training. Thus, the pre-synaptic input normalized by tdBN \cite{zheng2021going} may not satisfy $I \nsim N(0,(V_{th})^2)$, which has not been considered in many previous studies. When SG uses the fixed gradient-available interval, unpredictable shifts in the membrane potential will naturally deviate from the optimal areas for gradient matching, resulting in performance limitations. As the membrane potential is directly computed from the pre-synaptic input, clarifying the distribution of pre-synaptic input helps to analyze the membrane potential shifts. For this respect, we propose \textbf{Theorem 1} to rethink the specific distribution of pre-synaptic input.

\begin{theorem}
With the iterative LIF model and tdBN method, assuming normalized pre-synaptic input $I \sim N(0, (V_{th})^2)$, we have $\Bar{I} \sim N(\Bar{\beta}, (\Bar{\gamma} V_{th})^2)$ after affine transformation, where $\Bar{\beta} = \frac{1}{C} \sum_{c=1}^C \beta_c$ and $\Bar{\gamma} = \frac{1}{C} \sum_{c=1}^C\gamma_c$, $C$ is the channel size of tdBN layer.
\end{theorem}

\begin{proof}
The proof of \textbf{Theorem 1} is presented in \textbf{Supplementary \ref{appendix_proof}}.
\end{proof}

\textbf{Theorem 1} proves that the distribution of pre-synaptic input normalized by tdBN is not only governed by threshold $V_{th}$, but also by the parameters $\gamma_c$ and $\beta_c$ of affine transformations. As the foundation of our work, it provides a new insight into the distribution of pre-synaptic input.

\begin{figure}[htbp]
    \centering
    \includegraphics[width=0.95\linewidth]{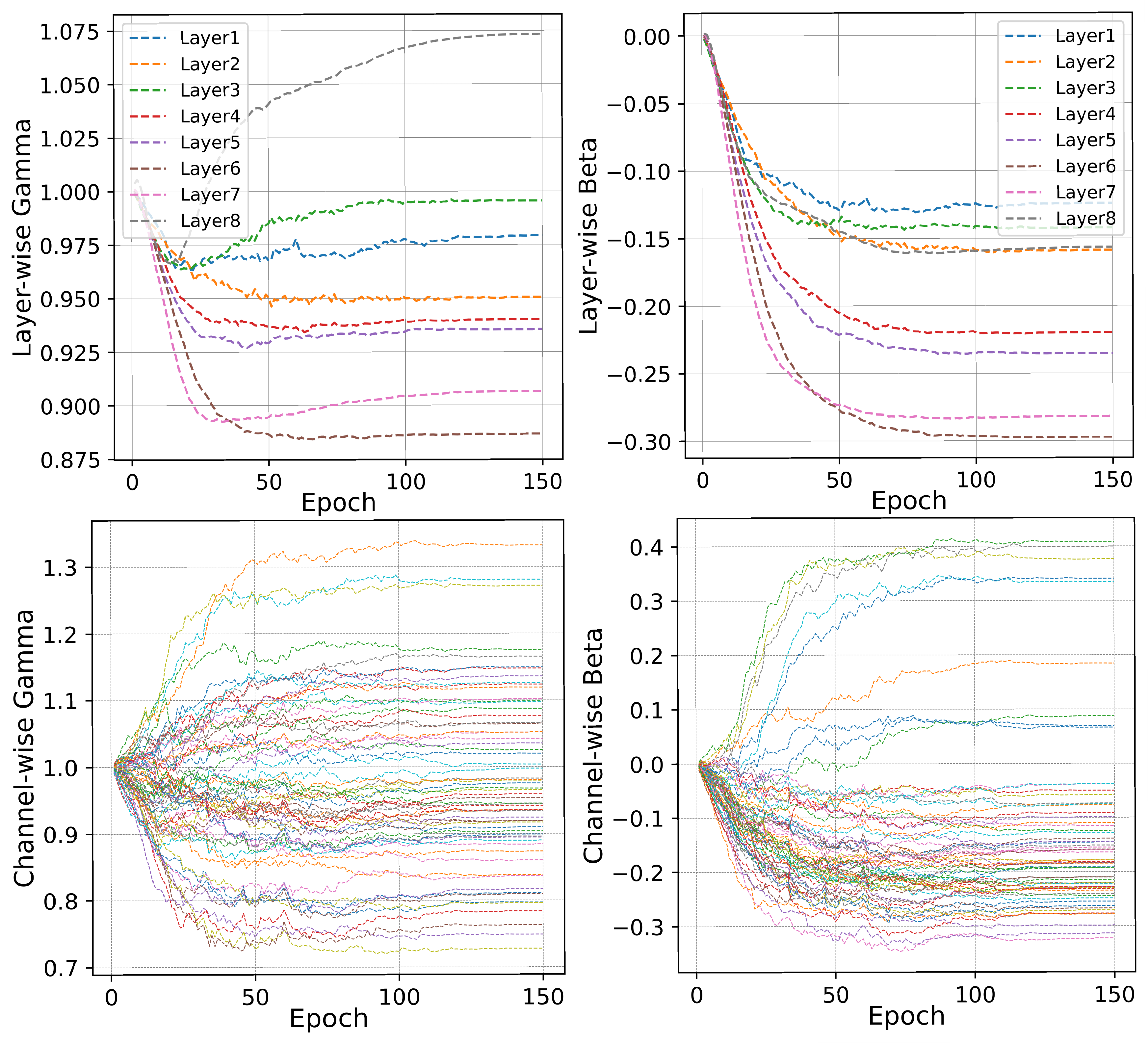}
    \caption{The affine transformation of tdBN in an 8-layer vanilla SNN. Top line is the variation curves of parameters $\gamma$ and $\beta$ for the average of all channels in each layer. Bottom line is the variation curves of parameters $\gamma$ and $\beta$ for all channels in the first layer.}
    \label{Fig::tdBN_param}
\end{figure}

\subsection{Adaptive Gradient Rule}
To optimize SNN learning, we need to synchronously modify the gradient-available interval of SG to better align with the evolving MPD. Thus, it is essential to analyze the detailed dynamics of membrane potentials. \cite{zheng2021going} derived a high degree of similarity between the distribution of pre-synaptic input and membrane potential in neurons. \cite{lian2023learnable} extended this reasoning that for a given pre-synaptic input $I \sim N(0, (V_{th})^2)$, then the distribution of membrane potential is only determined by decay factor $\tau$ and satisfies $V \sim N(0, (1+\tau^2)(V_{th})^2)$. Based on the analysis in Section \ref{Sec::rethinking_inputDistribution}, it is known that the pre-synaptic input will deviate from the desired distribution after tdBN normalization. To this end, we further propose \textbf{Theorem 2} to express the relation between the pre-synaptic input, the decay factor, and the membrane potential at different timesteps.

\begin{theorem} 
Consider an SNN with $T$ timesteps, the pre-synaptic input of neurons injected into the tdBN layer with affine transformation is normalized to satisfy $\Bar{I} \sim N(\Bar{\beta}, (\Bar{\gamma} V_{th})^2)$, we have the membrane potential $\Bar{V} \sim N(\Bar{\beta}, (\Bar{\gamma} V_{th})^2)$ when $t=1$, and $\Bar{V} \sim N((1+\tau)\Bar{\beta}, (1+\tau^2)(\Bar{\gamma} V_{th})^2)$ when $t>1$, where $t \in T$.
\end{theorem}

\begin{proof}
The proof of \textbf{Theorem 2} is also presented in \textbf{Supplementary \ref{appendix_proof}}.
\end{proof}

\textbf{Theorem 2} describes the dynamic distribution of membrane potential at different timesteps. \cite{lian2023learnable} observed a correlation between SG width $\kappa$ and neuron decay factor $\tau$ and manually designed a proportional function (e.g. $\kappa = f(\tau)$) to describe it. As the decay factor also affects the membrane potential (Eq. \ref{Eq::iterativeLIF_V}), the proportional function can link SG width to MPD. From that view, we can avoid gradient information loss or redundancy by dynamically adjusting the SG width to control the alignment between the gradient-available interval and the evolving MPD. Then, we will concentrate on how to design the correlation function.
 
For a well-formed $f(\cdot)$, the key is to control the proportion of MPD in SG within a reasonable level. The variance reflects the dispersion of a distribution. An increase in variance indicates a more dispersed MPD, so the width needs to be enlarged to increase the proportion of neurons in SG to avoid gradient information loss. Conversely, a decrease in variance requires narrowing the width to reduce the proportion. Thus, a positive correlation arises between SG width and MPD. Here, we empirically set $\kappa$ as 2 times the square root of $\mathbb{VAR}$ when $V_{th}=0.5$ in our work, which ensures the initial width satisfies the standard rectangular SG setting  (i.e. $k\approx 1$). Moreover, as PLIF neurons \cite{fang2021incorporating} can hierarchically learn the decay factor in SNNs, we also employ them to cooperate with the learnable affine transformation to control the evolving MPD together. It does not destroy the correlation function for scaling the SG width in response to MPD, but also enhances the expressiveness of SNNs. Finally, the correlation function $f(\cdot)$ can be formulated as
\begin{align}
    \kappa = f(\tau^n) &= \begin{cases} 2 \times (\Bar{\gamma}^n V_{th}), & t=1 \\
    2 \times \sqrt{1+(\tau^n)^2}(\Bar{\gamma}^n V_{th}), & t>1
    \end{cases} \\
    \tau^n &= sigmoid(\rho^n) = \frac{1}{1+e^{-\rho^n}},
\end{align}
where learnable $\rho^n$ is a layer-wise factor to ensure $\tau^n \in (0, 1)$. $\tau^n$ is initialized to 0.2 for all layers, which is adjusted when $\rho^n$ is updated based on gradients (Eq.~\ref{Eq::gradient_rho}). In this way, SG can accurately capture the membrane potential shift and promptly update the gradient-available interval, effectively optimizing the loss landscape of SNNs.

\subsection{The Overall Training Procedure}
Employing the iterative LIF neurons in SNN has temporal dynamics in the spatial domain, which can well apply the spatio-temporal backpropagation algorithm (STBP) \cite{wu2018spatio} to update synapse weights. In the readout layer, we also only accumulate the membrane potential of output neurons without leakage and firing, as did in recent works \cite{rathi2023diet,deng2022temporal}, which can be described by
\begin{align}
    o^N_i &= \frac{1}{T} \sum^T_{t=1} \sum_{j=1}^{l(N-1)} w^N_{ij}S^{N-1}_j(t),
\end{align}
where $N$ and $T$ denote the number of layers and timesteps, respectively. Then, the gradient of synaptic weights $w^n_{ij}$ and learnable $\rho^n$ can be derived by the chain rule:
\begin{align}
    \frac{\partial L}{\partial w^n_{ij}} &= \sum^T_{t=1} \frac{\partial L}{\partial V^n_i(t)} \frac{\partial V^n_i(t)}{\partial I^n_i(t)} \frac{\partial I^n_i(t)}{\partial w^n_{ij}} \nonumber \\
    &= \sum^T_{t=1} \frac{\partial L}{\partial V^n_i(t)} \sum^{l(n-1)}_{j=1} S^{n-1}_j(t). \\
    \frac{\partial L}{\partial \rho^n} &= \sum^T_{t=1} \frac{\partial L}{\partial V^n_i(t)} \frac{\partial V^n_i(t)}{\partial \tau^n} \frac{\partial \tau^n}{\partial \rho^n} \nonumber \\
    &= \sum^T_{t=1} \frac{\partial L}{\partial V^n_i(t)} \frac{\partial V^n_i(t)}{\partial \tau^n} \tau^n(1-\tau^n). \label{Eq::gradient_rho}
\end{align}

As $V^n_i(t)$ not only contributes to the $S^n_i(t)$ but also governs the $V^n_i(t+1)$, it can be derived by
\begin{align}
    \frac{\partial L}{\partial V^n_i(t)} &= \frac{\partial L}{\partial S^n_i(t)} \frac{\partial S^n_i(t)}{\partial V^n_i(t)} + \frac{\partial L}{\partial V^n_i(t+1)} \frac{\partial V^n_i(t+1)}{\partial V^n_i(t)}, \\
    \frac{\partial L}{\partial S^n_i(t)} &= \sum^{l(n+1)}_{j=1} \frac{\partial L}{\partial V^{n+1}_j(t)} \frac{\partial V^{n+1}_j(t)}{\partial S^n_j(t)} \nonumber \\
    & + \frac{\partial L}{\partial V^n_i(t+1)} \frac{\partial V^n_i(t+1)}{\partial S^n_i(t)}.
\end{align}
Moreover, the pseudocode of the overall training procedure is briefed in \textbf{Algorithm \ref{alg::algorithm1}}.

\begin{algorithm}[!ht]
\caption{The overall training procedure of SNNs with MPD-AGL algorithm in one iteration}
\textbf{Input:} Timestep: $T$; Threshold: $V_{th}$; Initial layer-wise decay: $\tau^n$; input: $S(t), t \in T$; true-label vector: $Y$. \\
\textbf{Output:} updated the weight $w_{ij}^n$ and learnable $\rho^n$of SNNs.
\begin{algorithmic}[1]
    \Statex \textbf{Forward:}
        \For{$n = 1$ to $N$}
        \If{$n < N$}
            \State Compute $I^n = w^n S^{n-1}$ // (1)
            \State $\Bar{I}^n \gets tdBN(I^n)$ // (4) and (5)
            \For{$t = 1$ to $T$}
                \State Compute $\Bar{V}^n(t), S^n(t)$ // (2) and (3)
                \State Compute the width of SG $\kappa$ // (7) and (8)
            \EndFor   
        \Else
            \State $o^N = \frac{1}{T} \sum^T_{t=1}(w^N S^{N-1}(t))$ // (9)
        \EndIf
        \EndFor
        \State $L \gets CrossEntropy(o^N,Y)$
    \Statex \textbf{Backward:}
        \For{$n = 1$ to $N$}
            \For{$t = 1$ to $T$}
                \State $\frac{\partial L}{\partial V^n(t)} \gets Grad(\frac{\partial L}{\partial S^n(t)},\frac{\partial L}{\partial V^n(t+1)})$ // (12)
                \State $\frac{\partial L}{\partial S^n(t)} \gets Grad(\frac{\partial L}{\partial V^{n+1}(t)},\frac{\partial L}{\partial V^n(t+1)})$ // (13)
            \EndFor
        \EndFor
    \State Update the parameters $w_{ij}^n$ and $\rho^n$. // (10) and (11)
\end{algorithmic}
\label{alg::algorithm1}
\end{algorithm}

\section{Experiment}
In this section, we evaluate SNN with MPD-AGL for classification tasks on static CIFAR10/100, Tiny-ImageNet datasets, and the neuromorphic CIFAR10-DVS dataset.
\subsection{Comparisons with Other Methods}
As listed in Table \ref{Tab::classification_accuracy}, we compare the classification accuracy of the proposed method with other advanced methods. For \textbf{CIFAR10} dataset, MPD-AGL with ResNet-19 achieves 96.54\% accuracy in 6 timesteps, significantly outperforming other methods. Notably, at ultra-low latency ($T=2$), our method even slightly improves over all compared methods. For \textbf{CIFAR100} dataset, MPD-AGL still performs well and achieves the best accuracy of 80.49\% in only 6 timesteps. Furthermore, our method outperforms LSG by an overwhelming margin of 2.52\%, 2.87\%, and 3.36\%, respectively. The main reason is that LSG neglects the effect of affine transformation on the pre-synaptic input and membrane potential. As a result, the LSG-designed learnable SG cannot accurately capture the evolving MPD. For \textbf{CIFAR10-DVS} dataset, MPD-AGL with VGGSNN in 10 timesteps can reach an accuracy of 84.10\% by using the TET loss \cite{deng2022temporal}. It even achieves the 82.50\% accuracy w/o it, which is a greater improvement over other methods. For \textbf{Tiny-ImageNet} dataset, MPD-AGL with VGG-13 achieves the accuracy of 58.14\% in 4 timesteps, outperforming ASGL by 1.57\%.

\subsection{Proportion of Gradient Available}
To investigate whether the proposed method can effectively alleviate the gradient vanishing problem, we conducted experiments using ResNet-19 on the CIFAR10 dataset with 2 timesteps. MPD-AGL rethinks the distribution of pre-synaptic input in the tdBN method \cite{zheng2021going} and, inspired by LSG \cite{lian2023learnable}, designs the correlation function to dynamically adjust SG. Therefore, we take STBP-tdBN and LSG as the benchmark algorithms. In Fig. \ref{Fig::gradient_available_rate}(a) and Fig. \ref{Fig::gradient_available_rate}(b), we compare the training loss and test accuracy of these three methods, where MPD-AGL can optimize the training loss to lower smooth values that have better generalization ability. Then we visualize the gradient-available proportion curve for layer 7 (Fig. \ref{Fig::gradient_available_rate}(c)). The fixed width of SG in STBP-tdBN cannot effectively match the evolving MPD, which causes many neurons to fall outside the gradient-available interval and slow weight updates. Compared with STBP-tdBN, LSG can slightly alleviate this situation, but cannot respond to MPD timely. Specifically, LSG takes more epochs to make the proportion of neurons fall into the gradient-available interval to an appropriate level. Obviously, MPD-AGL can quickly capture the shifts in membrane potential and respond promptly. To reveal how our method helps SNNs for gradient propagation, we display the width of SG and the proportion of neurons that fall into the gradient-available interval in each layer. As illustrated in Fig. \ref{Fig::gradient_available_rate}(d-e), the SG width in MPD-AGL is distributed mostly around 1.26, whereas LSG is 1.12. It means that MPD-AGL makes more neurons in deep layers have gradients, alleviating the gradient vanishing. Consequently, active neurons in all layers of MPD-AGL are higher than STBN-tdBN and LSG (Fig. \ref{Fig::gradient_available_rate}(f)).

\begin{table*}[!ht]
    \centering
    \caption{The comparison of classification performance on four benchmark datasets.}
    \resizebox{0.95\linewidth}{!}{
        \begin{tabular}{cccccc}		
        \hline
        \hline
        \textbf{Dataset}        & \textbf{Method}        & \textbf{SG optimization}     & \textbf{Architecture}     & \textbf{Timestep}     & \textbf{Accuarcy(\%)} \\                          
        \hline
        \multirow{8}*{CIFAR10}
            & TAB \cite{jiang2024tab}     & \usym{2717}    & ResNet-19      & 6 / 4 / 2     & 94.81 / 94.76 / 94.73 \\

            & ShortcutBP \cite{guo2024take}       & \usym{2717}   & ResNet-19        & 2     & 95.19 \\

            & STAtten + \cite{lee2025spiking}           & \usym{2717}    & SpikingReformer-6-384     & 4     & 95.26 \\
            
            & TCJA \cite{zhu2024tcja}       & \usym{2717}   & MS-ResNet-18     & 4     & 95.60 \\
            
            & PSG \cite{wang2025potential}          & \usym{2714}   & ResNet-19     & 6 / 4    & 95.00 / 95.12 \\
            
            & LSG \cite{lian2023learnable}          & \usym{2714}   & ResNet-19     & 6 / 4 / 2     & 95.52 / 95.17 / 94.41 \\
            
            & DeepTAGE \cite{liu2025deeptage} & \usym{2714}   & ResNet-18     & 4     & 95.86 \\
            
            & \textbf{Ours}          & \usym{2714}   & \textbf{ResNet-19}     & \textbf{6 / 4 / 2}     & \textbf{96.54 / 96.35 / 96.18} \\
    
        \hline
        \multirow{8}*{CIFAR100}
            & TAB \cite{jiang2024tab}       & \usym{2717}    & ResNet-19     & 6 / 4 / 2     & 76.82 / 76.81 / 76.31 \\
    
            & IM-LIF \cite{lian2024lif}       & \usym{2717}   & ResNet-19        & 6 / 3     & 77.42 / 77.21 \\

            & TCJA \cite{zhu2024tcja}       & \usym{2717}   & MS-ResNet-18     & 4     & 77.72 \\
            
            & STAtten + \cite{lee2025spiking}           & \usym{2717}    & SpikingReformer-6-384     & 4     & 77.90 \\
                 
            & PSG \cite{wang2025potential} & \usym{2714}   & ResNet-19     & 4     & 75.72 \\
            
            & LSG \cite{lian2023learnable}          & \usym{2714}   & ResNet-19     & 6 / 4 / 2     & 77.13 / 76.85 / 76.32 \\
          
            & ASGL \cite{wang2023adaptive}          & \usym{2714}   & ResNet-18     & 4 / 2   & 77.74 / 76.59 \\
            
            & \textbf{Ours}          & \usym{2714}   & \textbf{ResNet-19}     & \textbf{6 / 4 / 2}     & \textbf{80.49 / 79.72 / 78.84} \\
        
        \hline
        \multirow{8}*{CIFAR10-DVS}
            & IM-LIF \cite{lian2024lif}           & \usym{2717}    & VGGSNN      & 10     & 80.50 \\
            
            & IMPD-AGL \cite{Jiang2025tca}  & \usym{2717}    & VGGSNN       & 10     & 77.20 \\
            
            & TET \cite{deng2022temporal}           & \usym{2717}    & VGGSNN       & 10     & 77.33 / 83.17$^*$ \\

            & STAtten + \cite{lee2025spiking}           & \usym{2717}    & SpikingReformer-4-384     & 16     & 80.60 \\
            
            & PSG \cite{wang2025potential}    & \usym{2714}   & ResNet-19     & 7     & 76.00 \\
            
            & LSG \cite{lian2023learnable}  & \usym{2714}   & VGGSNN     & 10     & 77.90\\
            
            & DeepTAGE\cite{liu2025deeptage}     & \usym{2714}   & VGG-11     & 10      & 81.23 \\
    
            & \textbf{Ours}  & \usym{2714}   & \textbf{VGGSNN}     & \textbf{10}     & \textbf{82.50 / 84.10\bf{$^*$}}\\
        \hline
        \multirow{6}*{Tiny-ImageNet}
            & Offline LTL \cite{yang2022training} & \usym{2717}    & VGG-13        & 16     & 55.37 \\

            & S3NN \cite{suetake2023s3nn}           & \usym{2717}    &  ResNet-18     & 1     & 55.49 \\
           
            & IM-LIF \cite{lian2024lif}        & \usym{2717}    &  ResNet-19     & 6 / 3     & 55.37 / 54.82 \\

            & AT \cite{ozdenizci2024adversarially} & \usym{2717}    & VGG-11        & 8     & 57.21 \\
            
            & ASGL \cite{wang2023adaptive}  & \usym{2714}   & VGG-13     & 8 / 4     & 56.81 / 56.57\\
    
            & \textbf{Ours}  & \usym{2714}   & \textbf{VGG-13}     & \textbf{4}     & \textbf{58.14}\\
        \hline
        \hline
        \end{tabular}
    }
    \begin{tablenotes}
        \footnotesize
        \item $*$ denotes using the Adam optimizer with $lr = 1e-3$ and TET loss.
    \end{tablenotes}
    \label{Tab::classification_accuracy}
\end{table*}

\subsection{Effectiveness on SG Functions}
To clarify the effectiveness of our method on other SG functions, we conducted experiments using ResNet-19 on the CIFAR100 dataset with 2 timesteps. Here, we choose three other widely used SG functions, i.e., triangular \cite{bellec2018long}, sigmoid \cite{zenke2021remarkable}, and aTan \cite{fang2021incorporating}. Considering that optimal sharpness varies among different SGs, we scaled $\kappa$ proportionally. As shown in Fig. \ref{Fig::effectiveness_SG}, MPD-AGL achieves 76.43\% accuracy with triangular SG, outperforming STBP-tdBN and LSG by 2.68\% and 1.74\%, respectively. 
\begin{figure}[htbp]
    \centering
    \includegraphics[width=0.85\linewidth]{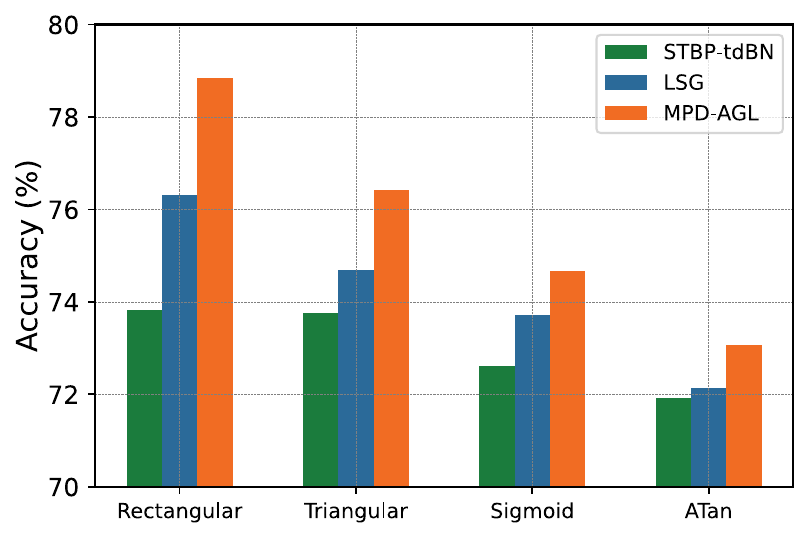}
    \caption{The effectiveness of other SG functions.}
    \label{Fig::effectiveness_SG}
\end{figure}
While MPD-AGL still performs better than STBP-tdBN and LSG on sigmoid and aTan SG functions, it does not perform as well as rectangular and triangular SG functions. This may be due to simply adjusting the sharpness of these asymptotic SG functions based on the evolving MPD, which leads to large oscillations in the gradient information. Instead, linear SG functions have a relatively smooth gradient estimation characteristic and thus exhibit stronger robustness.

\subsection{Energy Efficiency}
To validate the efficiency of SNNs in energy consumption, we conducted experiments using ResNet-19 on the CIFAR10 dataset. The theoretical energy consumption of SNNs can be estimated from synaptic operations (SOPs) \cite{zhou2023spikformer}. Due to the binarized and sparse nature of spikes, SNNs operate low-power AC operations only when neurons fire, and its required SOP varies with spike sparsity. In our model, real-valued images are directly fed into SNNs for encoding, and membrane potentials in the readout layer are used for prediction, so the SOPs contain AC operations and a few MAC operations. For the number of AC operations, we calculate it by $r^n \times T \times N^n_{AC}$, where $r^n$ is the average firing rate of $n$-th layer, $T$ is the timestep, and $N^n_{AC}$ is the number of AC operations in $n$-th layer of an iso-architecture ANN. For the number of MAC operations, it equals the $N_{MAC}$ of encoding and readout layers and scales by $T$ \cite{yao2023attention}. \cite{rathi2023diet} measured in 45 nm CMOS technology that an AC operation costs $0.9pJ$ and a MAC operation costs $4.6pJ$.

\begin{figure}[!ht]
    \centering
    \includegraphics[width=0.80\linewidth]{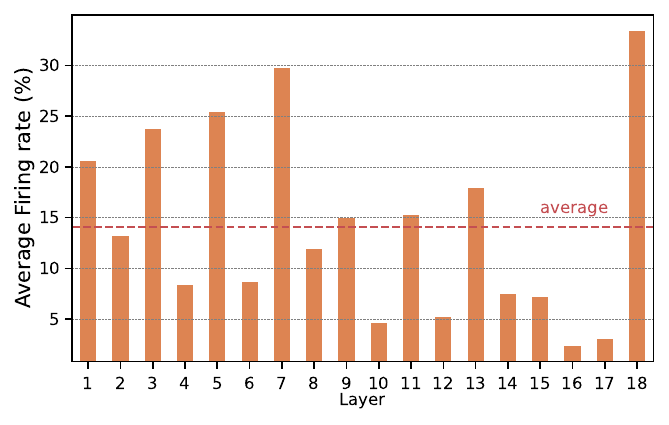}
    \caption{The average firing rate of each layer on CIFAR10 dataset.}
    \label{Fig::average_firing_rate}
\end{figure}

\begin{table}[!ht]
    \centering
    \caption{The energy consumption on the CIFAR10 dataset.}
    \resizebox{0.85\linewidth}{!}{
        \begin{tabular}{ccccc}
            \hline
            \hline
            \textbf{Method}      & \textbf{T}     & \textbf{\#Add.}       & \textbf{\#Multi.}      & \textbf{Energy} \\  
            \hline
            ANN     & -     & 2285.35M     & 2285.35M     & 10.51$mJ$ \\
            \hline
            STBP-tdBN   & 2 & 890.20M    & 7.08M    & 0.83$mJ$ \\
            \hline
            LSG   & 2 & 677.72M    & 7.08M    & 0.64$mJ$ \\
            \hline
            \multirow{3}*{\textbf{MPD-AGL}} 
            & \begin{tabular}[t]{@{}c@{}} 2 \\ 4 \\ 6 \end{tabular} 
            & \begin{tabular}[t]{@{}c@{}} 579.33M \\ 1004.70M \\ 1303.21M \end{tabular}
            & \begin{tabular}[t]{@{}c@{}} 7.08M \\ 14.16M \\ 21.25M \end{tabular}
            & \begin{tabular}[t]{@{}c@{}} 0.55$mJ$ \\ 0.96$mJ$ \\ 1.25$mJ$ \end{tabular} \\
            \hline
            \hline
        \end{tabular}
    }
    \label{Tab::energy_consumption}  
\end{table}

As shown in Fig. \ref{Fig::average_firing_rate}, the average firing rate of each layer in spiking ResNet-19 does not exceed 34\% (14\% on average) when $T=2$. In Table \ref{Tab::energy_consumption}, we estimate the energy consumption during inference at different timesteps, and the proposed MPD-AGL is 19$\times$ lower compared to ANN in 2 timesteps.

\begin{figure*}[htbp]
    \centering
    \includegraphics[width=0.91\linewidth]{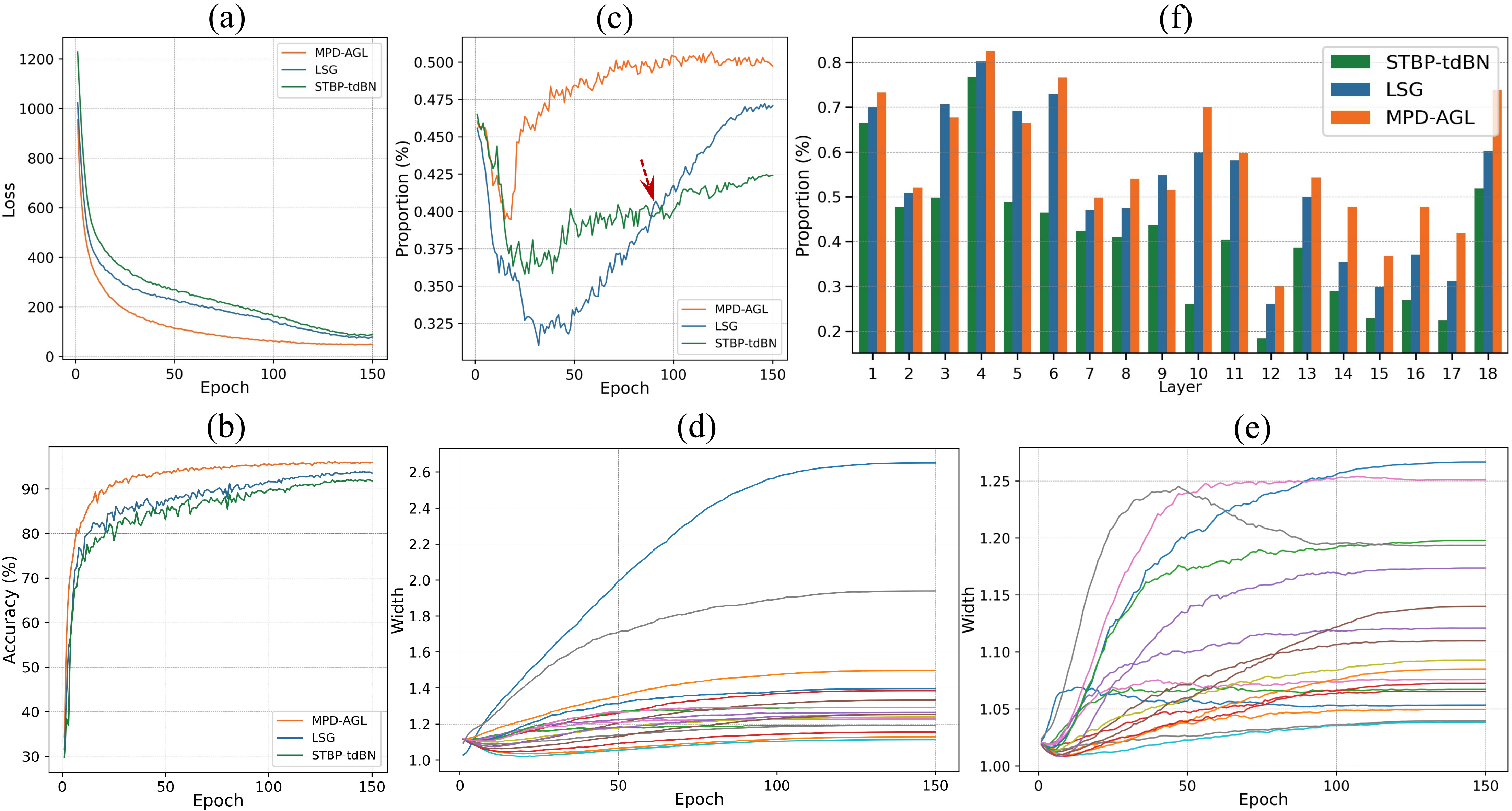}
    \caption{The comparison of different methods on the CIFAR10 dataset. (a) and (b) are the train loss and test accuracy, respectively. (c) and (f) are the proportion of neurons falling into the gradient-available interval in layer 7 and each layer of ResNet-19, respectively. (d) and (e) are the width of SG in each layer of MPD-AGL and LSG, respectively.}
    \label{Fig::gradient_available_rate}
\end{figure*}

\begin{figure}[htbp]
    \centering
    \includegraphics[width=0.95\linewidth]{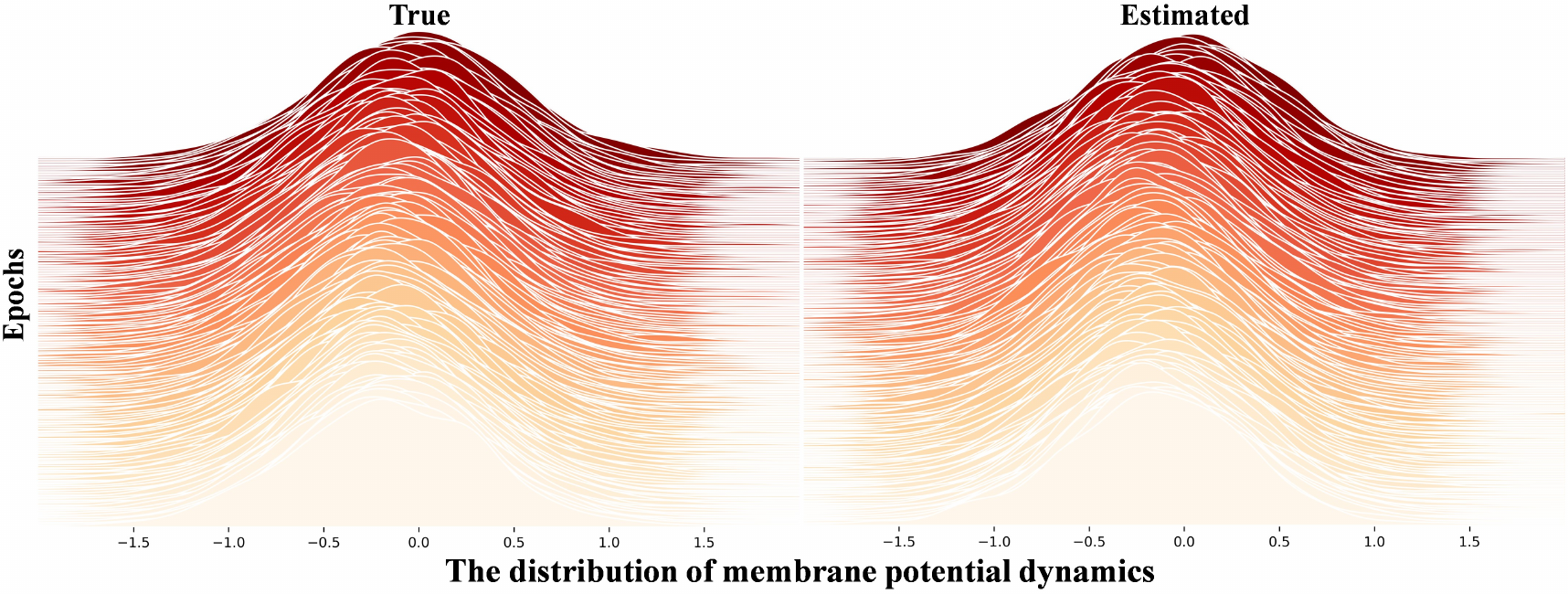}
    \caption{The proof of Theorem1 on CIFAR10-DVS dataset.}
    \label{Fig::ablation_proof_theorem}
\end{figure}

\begin{table}[!ht]
    \centering
    \caption{The ablation study on CIFAR10/100 dataset.}
    \resizebox{0.90\linewidth}{!}{
    \begin{tabular}{lccc}
        \hline
        \hline
        \multicolumn{1}{l}{\multirow{2}{*}{Method}} & \multicolumn{2}{c}{Accuracy ($\%$)} \\ 
        \cline{2-3} & CIFAR10  & CIFAR100 \\
        \hline
        Vanilla    & 92.38    & 73.87  \\
        w/ trainable decay & 93.15 & 74.12  \\
        w/ LSG \cite{lian2023learnable} & 94.41 & 76.32  \\
        w/ AGR & 95.93 & 78.47  \\
        \textbf{w/ MPD-AGL}   & \textbf{96.18} & \textbf{78.84} \\
        \hline
        \hline
    \end{tabular}
    }
    \begin{tablenotes}
        \footnotesize
        \item \textbf{AGR} denotes the proposed adaptive gradient rule
    \end{tablenotes}
    \label{Tab::ablation_accuracy}
\end{table}

\subsection{Ablation Study}
As shown in Fig. \ref{Fig::ablation_proof_theorem}, the true mean and variance of pre-synaptic input are close to the estimated values reasoned from Theorem 1 during training, proving the correctness of Theorem 1. It also indicates that the affine transformation of normalization layers is the reason for the membrane potential shifts, limiting the performance of SG learning. As for Theorem 2, which follows \cite{zheng2021going,lian2023learnable} by employing the factors of affine transformation, please refer to \textbf{Supplementary A} for detailed proofs. To evaluate the effectiveness of MPD-AGL algorithm, we also conducted experiments using ResNet-19 on the CIFAR10/100 datasets with 2 timesteps. In Table \ref{Tab::ablation_accuracy}, applying the proposed adaptive gradient rule (AGR) achieves an accuracy of 95.93\%/78.47\% on the CIFAR10/100 datasets, surpassing the vanilla and LSG methods by 3.55\%/4.60\% and 1.52\%/2.15\%, respectively. When combined with AGR and PLIF neurons, it even reaches 96.18\%/78.84\%, which means that the trainable decay can indeed combine with the adaptive gradient rule to enhance SNN learning.

\section{Conclusion}
In this work, we present a new perspective on understanding the gradient vanishing or mismatch problems in directly training SNNs with SG learning. We identify that these issues primarily arise as the failure of fixed SG and evolving MPD to align, which is caused by the affine transformation in normalization layers. Here, we propose the MPD-AGL algorithm, which adaptively relaxes SG in a temporal-aligned manner to more accurately capture the evolving MPD at different timesteps. Experimental results and theoretical analysis on four datasets demonstrate the effectiveness and superiority of our approach. MPD-AGL unlocks the limitation of SG width and provides more flexible gradient estimation for SNNs. Importantly, it can naturally integrate into existing SNN architectures to further enhance performance without additional inference costs, hopefully promoting the application of SNNs in more complex tasks and wider scenarios.

\bibliographystyle{named}
\bibliography{references}

\clearpage
\appendix
\onecolumn

\begin{center}
\section*{Supplementary Material}
\end{center}

\setcounter{theorem}{0}
\setcounter{equation}{13}
\section{Proofs of Theorems} \label{appendix_proof}
\begin{theorem}
With the iterative LIF model and tdBN method, assuming normalized pre-synaptic input $I \sim N(0, (V_{th})^2)$, we have $\Bar{I} \sim N(\Bar{\beta}, (\Bar{\gamma} V_{th})^2)$ after affine transformation, where $\Bar{\beta} = \frac{1}{C} \sum_{c=1}^C \beta_c$ and $\Bar{\gamma} = \frac{1}{C} \sum_{c=1}^C\gamma_c$, $C$ is the channel size of tdBN layer.
\end{theorem}

\begin{proof}
Perform an affine transformation of pre-synaptic inputs $I$, which gives
\begin{equation}
    \Bar{I} \sim
    \begin{cases}
    x_1 = N(\beta_1, (\gamma_1 V_{th})^2) \\
    x_2 = N(\beta_2, (\gamma_2 V_{th})^2) \\
    \quad \quad \quad \quad \vdots \\
    x_C = N(\beta_C, (\gamma_C V_{th})^2), 
    \end{cases}
\end{equation}
where $x_c \in \mathbb{R}^{N\times T\times H\times W}$ represents the pre-synaptic inputs at channel $c$ with $N$: batch axis, $T$: timestep axis, $(H, W)$: spatial axis. The expectation and variance of $\Bar{I}$, as follows \cite{duan2022temporal}:
\begin{align}
    \mathbb{E}[\Bar{I}] &= \frac{1}{C}(\beta_1 + \beta_2 + \cdots + \beta_C) = \frac{1}{C} \sum_{c=1}^C \beta_c, \\
    \mathbb{VAR}[\Bar{I}] &= \frac{1}{C}(\gamma_1^2 + \gamma_2^2 + \cdots + \gamma_C^2) (V_{th})^2.
\end{align}
According to the Cauchy-Schwarz Inequality, we have
\begin{equation}
    \frac{1}{C}(\gamma_1^2 + \gamma_2^2 + \cdots + \gamma_C^2) \ge (\frac{\gamma_1 + \gamma_2 + \cdots + \gamma_C}{C})^2 = (\frac{1}{C} \sum_{c=1}^C \gamma_c)^2.
\end{equation}
Finally, by initializing the pair of parameters $\gamma$ and $\beta$ with 1 and 0, we can get $\mathbb{E}[\Bar{I}] = \Bar{\beta}$ and $\mathbb{VAR}[\Bar{I}] \approx (\Bar{\gamma} V_{th})^2$, i.e., $\Bar{I} \sim N(\Bar{\beta}, (\Bar{\gamma} V_{th})^2)$, where $\Bar{\beta} = \frac{1}{C} \sum_{c=1}^C \beta_c$ and $\Bar{\gamma} = \frac{1}{C} \sum_{c=1}^C\gamma_c$, $C$ is the channel size of tdBN layer.
\end{proof}

\begin{theorem}
Consider an SNN with $T$ timesteps, the pre-synaptic input of neurons injected into the tdBN layer with affine transformation is normalized to satisfy $\Bar{I} \sim N(\Bar{\beta}, (\Bar{\gamma} V_{th})^2)$, we have the membrane potential $\Bar{V} \sim N(\Bar{\beta}, (\Bar{\gamma} V_{th})^2)$ when $t=1$, and $\Bar{V} \sim N((1+\tau)\Bar{\beta}, (1+\tau^2)(\Bar{\gamma} V_{th})^2)$ when $t>1$, where $t \in T$.
\end{theorem}

\begin{proof} 
Assuming the last firing time of $t'$, the membrane potential of iterative LIF model at $t$-th timestep can be expressed by
\begin{equation}
    \Bar{V}(t)=\sum_{k=t'}^t\tau^{t-k}\Bar{I}(k),
\end{equation}
where $\Bar{I}$ denotes the pre-synaptic input after affine transformation. $\tau$ is the decay factor of LIF models. In our SNN model, we set all $\tau$ to 0.2. When $t=1$, neurons are in the resting state, $\Bar{V}$ equals the pre-synaptic input at the initial moment. When $t>1$, $\Bar{V}$ equals the residual membrane potential of the previous moment plus the pre-synaptic input at the current moment. Then, we have

\begin{equation}\label{equation25}
    \Bar{V}(t) \approx \begin{cases}\Bar{I}(t),  &t=1\\ 
    \tau \Bar{I}(t-1)+\Bar{I}(t), &t >1\end{cases}
\end{equation}

According to \textbf{Theorem 1}, $\Bar{I}$ can be assumed as $i.i.d$ sample from $N(\Bar{\beta}, (\Bar{\gamma} V_{th})^2)$ \cite{zheng2021going}. Based on Eq. \ref{equation25}, we can express the mean and variance of $\Bar{V}$ as
\begin{equation}
\begin{aligned}
    \mathbb{E}[\Bar{V}] &\approx \begin{cases}\mathbb{E}[\Bar{I}] = \Bar{\beta},  &t=1\\ (1+\tau)\mathbb{E}[\Bar{I}] = ( 1+\tau) \Bar{\beta}, &t>1\end{cases}
\end{aligned}
\end{equation}
\begin{equation}
\begin{aligned}
    \mathbb{VAR}[\Bar{V}] &\approx \begin{cases}\mathbb{VAR}[\Bar{I}] = (\Bar{\gamma} V_{th})^2, &t=1\\ (1 + \tau^2)\mathbb{VAR}[\Bar{I}] = (1+\tau^2)(\Bar{\gamma} V_{th})^2, &t>1\end{cases}
\end{aligned}
\end{equation}
Finally, we can get the membrane potential $\Bar{V} \sim N(\Bar{\beta}, (\Bar{\gamma} V_{th})^2)$ when $t=1$, and $\Bar{V} \sim N((1+\tau)\Bar{\beta}, (1+\tau^2)(\Bar{\gamma} V_{th})^2)$ when $t>1$, where $t \in T$.
\end{proof}

\section{Experiment}
\subsection{Time Efficiency}
To evaluate the time overhead introduced by our method, we conducted experiments using ResNet19 on the CIFAR10 dataset. Table~\ref{Tab::time_efficiency} compares the running time of one epoch at 2 timesteps and shows that MPD-AGL does not introduce excessive extra overhead. As the inference process does not need gradient estimation, MPD-AGL mainly performs the adaptive SG width adjustment based on the MPD of different timesteps during the training process. In one iteration, the SG width adjustment consists of only $T$ sum-average operations ($\bar{\gamma}$) and $T$ multiplication operations (Eq. 7). For training time, MPD-AGL takes 3s more than STBP-tdBN, and this time overhead is mainly consumed in computing the affine transform parameters of normalization layers and the adaptive update of SG width. MPD-AGL takes 1s more than LSG, since LSG uses the same SG width for all timesteps, whereas our method needs to calculate and update the SG widths for different timesteps. As for inference time, MPD-AGL does not increase the time overhead.

\begin{table}[htbp]  
    \centering
    \caption{The comparison of time efficiency.}
    \renewcommand\arraystretch{1.1}
    \resizebox{0.5\linewidth}{!}{
        \begin{tabular}{cccc}
            \hline
            \hline
            {Methods} & tdBN & LSG & Ours\\  
            \hline
            {Training time} & 1m37s & 1m39s & 1m40s\\
            \hline
            {Inference time} & 8s & 8s & 8s\\
            \hline
            \hline
        \end{tabular}
    }
    \label{Tab::time_efficiency}
\end{table}

\subsection{Robustness}
As the core elements characterize the temporal dynamics of neurons, the firing threshold $V_{th}$ controls the sensitivity of neurons to input signals, and the decay factor $\tau$ affects the duration of neuronal excitation. To investigate the robustness of SNNs in different thresholds and different decay factors, we conducted experiments using ResNet19 on the CIFAR100 dataset with 2 timesteps. In Fig.~\ref{Fig::robustness_LIFparam}, STBP-tdBN and LSG perform poorly with different thresholds, e.g., with the same initial value of decay factors (set to 0.2), LSG decreased by 0.96\% and STBP-tdBN decreased by 2.71\% when the threshold increases from 0.5 to 1.0. The dependence on initial values is reduced by incorporating learnable decay factors in MPD-AGL and LSG. That is not for STBP-tdBN, e.g., when the threshold is 1.0, the accuracy with decay factors of 0.5 and 0.2 is 73.26\% and 71.32\%, respectively, which decreases by 1.94\%. Notably, MPD-AGL exhibits robustness to both thresholds and decay factors. In addition, we also show the performance of these methods in other widely-used architectures (Fig.~\ref{Fig::robustness_network}), and it is demonstrated that MPD-AGL still outperforms LSG and STBP-tdBN.

\begin{figure}[htbp]
    \centering
    \begin{subfigure}[]{0.49\linewidth}
        \includegraphics[width=\linewidth]{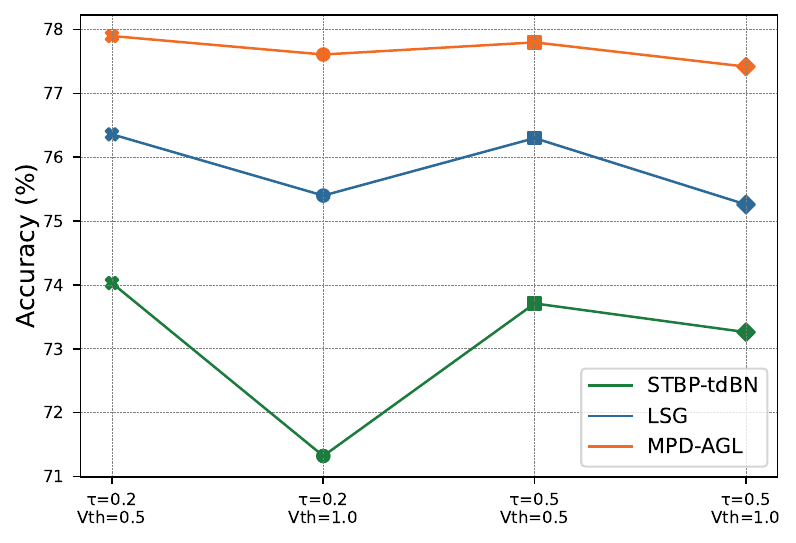}
        \caption{}
        \label{Fig::robustness_LIFparam}
    \end{subfigure}
    \begin{subfigure}[]{0.49\linewidth}
        \includegraphics[width=\linewidth]{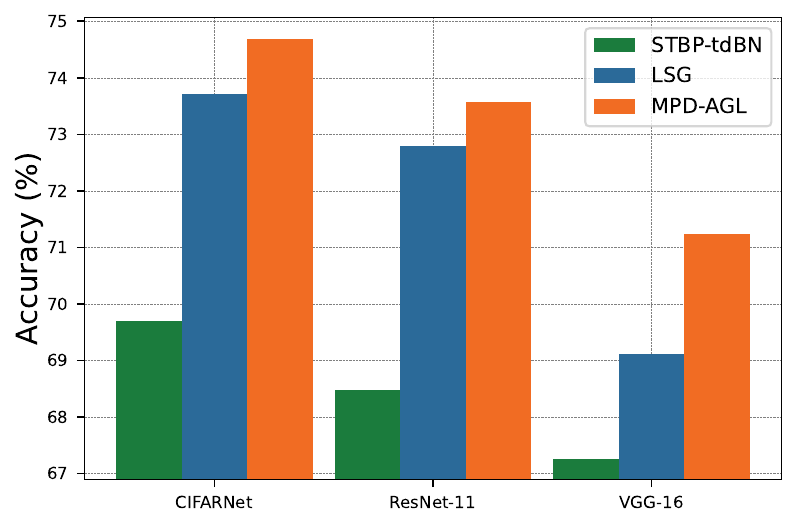}
        \caption{}
        \label{Fig::robustness_network}
    \end{subfigure}
    \caption{(a) The robustness with different neuron coefficients on the CIFAR100 dataset. (b) The robustness with different network structures on the CIFAR100 dataset.}
    \label{Fig::robustness}
\end{figure}

\newpage
\section{Experiments Details}
\subsection{Environment and Hyperparameter Settings}
All experiments are performed on a workstation equipped with Ubuntu 20.04.5 LTS, one AMD Ryzen Threadripper 3960X CPU running at 2.20GHz with 128GB RAM, and one NVIDIA GeForce RTX 3090 GPU with 24GB DRAM. The code is implemented in the Pytorch framework with version 3.9 of Python, and the weights are initialized randomly by the default method of Pytorch 1.12.1.

Table \ref{Tab::hyperparameters} lists the hyperparameters used in our work. SGD optimizer with an initial learning rate $lr = 0.1$, 0.9 momentum, and weight decay $1e-4$ is used in all datasets. All experiments used the CosineAnnealingLR scheduler to adjust $lr$, which will cosine decay to 0 over epochs.

\begin{table}[htbp]  
    \centering
    \caption{Hyperparmeter Settings.}
    \renewcommand\arraystretch{1.1}
    \resizebox{0.6\linewidth}{!}{
        \begin{tabular}{ccccc}
            \hline
            \hline
            Hyperparameters           &CIFAR10 &CIFAR100 &CIFAR10-DVS &Tiny-ImageNet\\  
            \hline
            $V_{th}$                 & 0.5           & 0.5       & 0.5    & 0.5    \\  
            $\tau$                 &  0.2          &  0.2       &  0.2    & 0.2    \\  
            Epoch           &  150           & 150        &  150    &  150    \\  
            Batch Size                &  100           &  100       &   50   &  100    \\
            Optimizer                 &  SGD          &  SGD       &  SGD    &  SGD   \\  
            $lr$             & 0.1           & 0.1        & 0.1   & 0.1    \\  
            \hline
            \hline  
        \end{tabular}
    }
    \label{Tab::hyperparameters}
\end{table}

\subsection{Datasets and Preprocessing}
\textbf{CIFAR-10:} The CIFAR-10 dataset \cite{krizhevsky2009learning} consists of 60,000 RGB static images across 10 classes, each with a 32 $\times$ 32 pixels resolution. These images are split into 50,000 for training and 10,000 for testing. In data preprocessing, we normalized the dataset by subtracting the global mean value of pixel intensity and dividing by the standard variance of RGB channels. Random Horizontal Flip and Crop were also applied to each image. AutoAugment \cite{cubuk2019autoaugment} was used for data augmentation.

\textbf{CIFAR-100:} The CIFAR-100 dataset \cite{krizhevsky2009learning} also contains 60,000 RGB static images with a resolution of 32 $\times$ 32 pixels in 100 classes, which are split into 50,000 training images and 10,000 test images. We adopt the same preprocessing and data augmentation strategy to the CIFAR-100 dataset as the CIFAR-10 dataset.

\textbf{CIFAR10-DVS:} The CIFAR10-DVS dataset \cite{li2017cifar10} is converted from 10,000 CIFAR10 images and is the most challenging mainstream neuromorphic dataset. It consists of 10 classes, each with 1,000 samples and a resolution of 128 $\times$ 128 pixels, but was not split into training and test sets. Following \cite{zheng2021going,samadzadeh2023convolutional}, we also used 90\% of the samples in each class for training and the rest 10\% for testing. In data preprocessing, we reduced the temporal resolution by segmenting the event stream into 10 temporal blocks, accumulating spike events within each block, and resizing the spatial resolution to 48 $\times$ 48 \cite{deng2022temporal,lian2023learnable}. Random Horizontal Flip and Random Roll within 5 pixels were applied for data augmentation \cite{li2022neuromorphic}.

\textbf{Tiny-ImageNet:} The Tiny-ImageNet dataset is the modified subset of the original ImageNet dataset \cite{deng2009imagenet}, which is more challenging than static CIFAR family datasets. It consists of 100,000 RGB static images for training and 10,000 RGB static images for testing across 100 classes, each with a 64 $\times$ 64 pixels resolution. In data preprocessing, we normalized the dataset by subtracting the global mean value of pixel intensity and dividing by the standard variance of RGB channels. Random Horizontal Flip and Crop were also applied to each image. AutoAugment \cite{cubuk2019autoaugment} was used for data augmentation.

\subsection{Network Architectures}
We adopt the widely-used ResNet-19 \cite{zheng2021going}, VGGSNN \cite{deng2022temporal}, and VGG-13 \cite{wang2023adaptive} network structures.The details of the network architecture are listed in Table \ref{Tab::ResNet-19}, Table \ref{Tab::VGGSNN} and Table \ref{Tab::VGG-13}, respectively. $xCy$ represents a convolutional layer with output channels equal to $x $, $kernel~size = y$, $stride~set = 1$, and $padding = 1$. $xFC$ represents a fully-connected layer with output features equal to $x$. $2AP$ represents the average-pooling layer with $kernel~size = 2$ and $stride = 2$. $z$ is the number of classes.

\begin{table}[htbp]
    \centering  
    \caption{ResNet-19 structures.}
    \renewcommand\arraystretch{1.1}
    \resizebox{0.4\linewidth}{!}{
        \begin{tabular}{cc}
        \hline\hline
        layer & ResNet-19 \\
        conv1 & 128C3 \\
        \hline
        block1 & $\left( \begin{array}{c}128C3\\128C3\end{array}\right)\times3$ \\
        \hline
        block2&$\left( \begin{array}{c}256C3\\256C3\end{array}\right)^*\times3$ \\
        \hline
        block3&$\left( \begin{array}{c}512C3\\512C3\end{array}\right)^*\times3$ \\
        \hline
        & average pool, 256-d FC \\
        &  10(11)-d FC \\
        \hline\hline
        \end{tabular}
    }
    \begin{tablenotes}
      \footnotesize
      \item * means the first basic block in the series performs downsampling directly with convolution kernels and a stride of 2.
    \end{tablenotes}
    \label{Tab::ResNet-19}
\end{table}

\begin{table}[htbp]
    \centering  
    \caption{VGGSNN structures.}
    \renewcommand\arraystretch{1.1}
    \resizebox{0.5\linewidth}{!}{
        \begin{tabular}{cc}
        \hline\hline
        \multirow{6}{*}{VGGSNN} & 64C3-LIF \\
            & 128C3-LIF-2AP \\ 
            & 256C3-LIF-256C3-LIF-2AP \\
            & 512C3-LIF-512C3-LIF-2AP \\
            & 512C3-LIF-512C3-LIF-2AP \\
            & -$z$FC \\
        \hline\hline
        \end{tabular}
    }
    \label{Tab::VGGSNN}
\end{table}

\begin{table}[htbp]
    \centering  
    \caption{VGG-13 structures.}
    \renewcommand\arraystretch{1.1}
    \resizebox{0.5\linewidth}{!}{
        \begin{tabular}{cc}
        \hline\hline
        \multirow{6}{*}{VGG-13} & 64C3-LIF-64C3-LIF-2AP \\
            & 128C3-LIF-128C3-LIF-2AP \\ 
            & 256C3-LIF-256C3-LIF-2AP \\
            & 512C3-LIF-512C3-LIF-2AP \\
            & 512C3-LIF-512C3-LIF-2AP \\
            & -4096FC-LIF-4096FC-LIF-$z$FC \\
        \hline\hline
        \end{tabular}
    }
    \label{Tab::VGG-13}
\end{table}

\end{document}